%% file: AAAI-GroverA.6849.tex
\def\year{2020}\relax
\documentclass[letterpaper]{article} 
\usepackage{aaai20}  
\usepackage{times}  
\usepackage{helvet} 
\usepackage{courier}  
\usepackage[hyphens]{url}  
\usepackage{graphicx} 
\usepackage{graphicx}  
\newcommand{\name}{AlignFlow}
\title{AlignFlow: Cycle Consistent Learning from\\ Multiple Domains via Normalizing Flows}

\usepackage[utf8]{inputenc} 
\usepackage[T1]{fontenc}    
\usepackage{url}            
\usepackage{booktabs}       
\usepackage{microtype}      
\usepackage{subcaption}
\usepackage{multirow}
\usepackage[toc,page]{appendix}

\usepackage{amsmath, amsfonts, bbm, amsthm, nicefrac,comment}
\newtheorem{theorem}{Theorem}

\newtheorem{definition}{Definition}

\newtheorem{lemma}{Lemma}

\newtheorem{proposition}{Proposition}

\newcommand{\namecite}[1]{\citeauthor{#1}~\shortcite{#1}}

\newcommand\citep{\cite}
\newcommand\citet{\namecite}

\author{Aditya Grover\thanks{Equal Contribution}, Christopher Chute$^\ast$, Rui Shu, Zhangjie Cao, Stefano Ermon\\
Department of Computer Science\\
Stanford University\\
\{adityag,chute,ruishu,caozj,ermon\}@cs.stanford.edu 
}

\usepackage{tikz,tkz-graph,subcaption}
\usetikzlibrary{
  matrix,
  graphs,
  graphs.standard,
  positioning,chains,fit,shapes,calc,
  decorations.pathreplacing
}

\usepackage{amsmath,amsfonts,bm}

\def\eqref#1{equation~\ref{#1}}

\def\1{\bm{1}}

\def\calX{{\mathcal{X}}}
\def\calY{{\mathcal{Y}}}
\def\calZ{{\mathcal{Z}}}

\def\calC{{\mathcal{C}}}
\def\calD{{\mathcal{D}}}

\def\calG{{\mathcal{G}}}
\def\calS{{\mathcal{S}}}
\def\gab{G_{\textnormal{A} \rightarrow  \textnormal{B}}}
\def\gba{G_{\textnormal{B} \rightarrow  \textnormal{A}}}

\def\gza{G_{\textnormal{Z} \rightarrow  \textnormal{A}}}
\def\gaz{G_{\textnormal{A} \rightarrow  \textnormal{Z}}}
\def\gzb{G_{\textnormal{Z} \rightarrow  \textnormal{B}}}
\def\gbz{G_{\textnormal{B} \rightarrow  \textnormal{Z}}}

\def\gyx{G_{\textnormal{Y} \rightarrow  \textnormal{X}}}

\def\ildj#1#2#3{\left \vert \mathrm{det} \frac{\partial#1^{-1}}{\partial#2}  \right\vert_{#2=#3}}

\def\ra{{\textnormal{A}}}
\def\rb{{\textnormal{B}}}

\def\rx{{\textnormal{X}}}
\def\ry{{\textnormal{Y}}}
\def\rz{{\textnormal{Z}}}

\DeclareMathAlphabet{\mathsfit}{\encodingdefault}{\sfdefault}{m}{sl}
\SetMathAlphabet{\mathsfit}{bold}{\encodingdefault}{\sfdefault}{bx}{n}

\newcommand{\E}{\mathbb{E}}

\begin{document}
\maketitle

\begin{abstract}
Given datasets from multiple domains, a key challenge is to efficiently exploit these data sources for modeling a target domain.
Variants of this problem have been studied in many contexts, such as cross-domain translation and domain adaptation.
We propose \name{}, a generative modeling framework that models each domain via a normalizing flow.
The use of normalizing flows allows for a) flexibility in specifying learning objectives via adversarial training, maximum likelihood estimation, or a hybrid of the two methods; and b) learning and exact inference of a shared representation in the latent space of the generative model. 
We derive a uniform set of conditions under which \name{} is marginally-consistent for the different learning objectives.
Furthermore, we show that \name{} guarantees exact cycle consistency in mapping datapoints from a source domain to target and back to the source domain.
Empirically, \name{} outperforms relevant baselines on image-to-image translation and unsupervised domain adaptation and can be used to simultaneously interpolate across the various domains using the learned representation. 
\end{abstract}

\section{Introduction}

In recent years, there has been an increase in the availability of both labeled and unlabeled datasets from multiple sources.
For example, many variants of face datasets scraped from sources such as Wikipedia and IMDB are publicly available.
Given data from two or more domains, we expect sample-efficient learning algorithms to be able to learn and \textit{align} the shared structure across these domains for accurate downstream tasks.
This perspective has a broad range of applications across machine learning, including
relational learning~\citep{kim2017learning}, 
domain adaptation~\citep{taigman2016unsupervised,cycada,bousmalis2017unsupervised},
image and video translation for computer vision~\citep{isola2017image,wang2018video}, and machine translation for low resource languages~\citep{gu2018universal}.

Many variants of the domain alignment problem have been studied in prior work.
For instance, \textit{unpaired cross-domain translation} refers to the task of learning a mapping from one domain to another given datasets from the two domains~\citep{zhu2017toward}. 
This task can be used as a subproblem in the \textit{domain adaptation} setting, where the goal is to learn a classifier for the unlabeled domain given labeled data from a related source domain~\citep{saenko2010adapting}.
Many of these problems are underconstrained due to the limitations on the labelled supervision available for target domain.
An amalgam of inductive biases need to be explicitly enforced to learn meaningful solutions, e.g., cycle-consistency~\citep{zhu2017toward}, entropic regularization~\citep{courty2017joint} etc.
These inductive biases can be specified via additional loss terms or by specifying constraints on the model family.

We present \name{}, a latent variable generative framework that seeks to discover the shared structure across multiple data domains using normalizing flows~\citep{normalizing-flow,nice,real-nvp}. 
Latent variable generative models are highly effective for inferring hidden structure within observed data from a single domain.
In \name{}, we model the data from each domain via an invertible generative model with \emph{a single latent space shared across all the domains}. 
If we let the two domains to be $\ra$ and $\rb$ with a shared latent space, say $\rz$, then the latent variable generative model for $\ra$ may additionally share some or all parameters with the model of domain $\rb$.
Akin to a single invertible model, the collection of invertible models in \name{} provide great flexibility in specifying learning objectives and can be trained via 
maximum likelihood estimation, adversarial training or a hybrid variant accounting for both objectives.

By virtue of an invertible design, \name{} naturally extends as a cross-domain translation model.
To translate data across two domains, say $\ra$ to $\rb$, we can invert a data point from $\ra \to \rz$ first followed by a second inversion from $\rz \to \rb$.
Appealingly, we show that this composition of invertible mappings is \textit{exactly cycle-consistent}, i.e., translating a datapoint from  $\ra$ to $\rb$ using the forward mapping and backwards using the reverse mapping gives back the original datapoint and vice versa from $\rb$ to $\ra$.
Cycle-consistency was first introduced in CycleGAN~\citep{cycle-gan} and has been shown to be an excellent inductive bias for underconstrained problems, such as unpaired domain alignment. 
While models such as CycleGAN only provide approximate cycle-consistency by incorporating additional loss terms, \name{} can omit these terms and guarantee exact cycle-consistency by design.

We analyze the \name{} framework extensively. 
Theoretically, we derive conditions under which the \name{} objective is consistent in the sense of recovering the true marginal distributions.
For objectives that use adversarial loss terms, we derive optimal critics in this setting.
Empirically, we consider two sets of tasks: image-to-image translation and unsupervised domain adaptation.
On both these tasks, we observe consistent improvements over other approximately cycle-consistent generative frameworks on three benchmark pairs of high-dimensional image datasets.
\section{Preliminaries}

In this section, we discuss the necessary background and notation on generative adversarial networks and normalizing flows.
We overload uppercase notation $\rx, \ry, \rz$ to denote random variables and their sample spaces and use lowercase notation $x, y, z$ to denote values assumed by these variables.

\subsection{Generative Adversarial Networks}
Generative adversarial networks (GAN) are a class of latent variable generative models that specify the generator as a deterministic mapping $h:\rz \rightarrow \rx$ between a set of latent variables $\rz$ and a set of observed variables $\rx$~\citep{gan}. 
In order to sample from a GAN, we need a prior density over $\rz$ that permits efficient sampling. The generator of a GAN can also be conditional, where the conditioning is on another set of observed variables and optionally the latent variables $\rz$ as before~\citep{mirza2014conditional}.

A GAN is trained via adversarial training, wherein the generator $h$ plays a minimax game with an auxiliary critic $C$. The goal of the critic $C: \rx \rightarrow \mathbb{R}$ is to distinguish real samples from the observed dataset with samples generated via $h$.
The generator, on the other hand, tries to generate samples that can maximally \textit{confuse} the critic. 
Many learning objectives have been proposed for adversarial training, such as those based on f-divergences~\citep{nowozin2016f} and Wasserstein Distance~\citep{arjovsky2017wasserstein}. 
For the standard cross-entropy GAN, the critic outputs a probability of a datapoint being real and optimizes the following objective w.r.t. a data distribution $p^\ast_\rx:\rx \rightarrow \mathbb{R}_{\geq 0}$:
\begin{align}\label{eq:adv}
  \mathcal{L}_{\textnormal{GAN}}(C, h) &= \E_{x \sim p^\ast_\rx}[\log C(x)] \nonumber \\
  &+ \E_{z \sim p_\rz}[\log(1- C(h(z)))].
\end{align}
for a suitable choice of prior density $p_\rz$.
The generator and the critic are both parameterized by deep neural networks and learned via alternating gradient updates. 
Because adversarial training only requires samples from the generative model, it can be used to train generative models with intractable or ill-defined likelihoods~\citep{mohamed2016learning}.
Hence, adversarial training is \textit{likelihood-free} and in practice, it gives excellent performance for tasks that require data generation.
However, these models are hard to train due to the alternating minimax optimization and suffer from issues such as mode collapse~\citep{goodfellow2016nips}.

\subsection{Normalizing Flows}
Normalizing flows are a class of latent variable generative models that specify the generator as an \emph{invertible} mapping $h:\rz \rightarrow \rx$ between a set of latent variables $\rz$ and a set of observed variables $\rx$. 
Let $p_\rx$ and $p_\rz$ denote the marginal densities defined by the model over $\rx$ and $\rz$ respectively.
Using the change-of-variables formula, these marginal densities can be related as:
\begin{align}\label{eq:cov}
p_\rx(x) = p_\rz(z) \ildj{h}{\rx}{x}
\end{align}
where $z=h^{-1}(x)$ due to the invertibility constraints. 
Here, the second term on the RHS corresponds to the absolute value of the determinant of the Jacobian of the inverse transformation and signifies the change in volume when translating across the two sample spaces.

For evaluating likelihoods via the change-of-variables formula, we require efficient and tractable evaluation of the prior density, the inverse transformation $h^{-1}$, and the determinant of the Jacobian of $h^{-1}$. 
To draw a sample from this model, we perform ancestral sampling, i.e., we first sample a latent vector $z \sim p_\rz(z)$ and obtain the sampled vector as given by $x = h(z)$. 
This requires the ability to efficiently: (1) sample from the prior density and (2) evaluate the forward transformation $h$.
Many transformations parameterized by deep neural networks that satisfy one or more of these criteria have been proposed in the recent literature on normalizing flows, e.g., NICE~\citep{nice} and Autoregressive Flows~\citep{iaf,maf}. 
By suitable design of transformations, both likelihood evaluation and sampling can be performed efficiently, as in Real-NVP~\citep{real-nvp}.
Consequently, a flow model can be trained efficiently to maximize the likelihood of the observed dataset (a.k.a. maximum likelihood estimation or MLE) as well as likelihood-free adversarial training~\citep{flow-gan}. 

\section{The \name{} Framework}
 
In this section, we present the \name{} framework for learning generative models in the presence of unpaired data from multiple domains.
For ease of presentation, we consider the case of two domains.
Unless mentioned otherwise, our results naturally extend to more than two domains as well.

\subsection{Problem Setup}
The learning setting we consider is as follows.
We are given unpaired datasets $\calD_\ra$ and $\calD_\rb$ from two domains $\ra$ and $\rb$ respectively. We assume that the datapoints are sampled i.i.d. from some true but unknown marginal densities denoted as $p^\ast_\ra$ and $p^\ast_\rb$ respectively. 
We are interested in learning models for the following distributions:  (a) the marginal likelihoods $p_\ra$ and $p_\rb$ that approximate $p^\ast_\ra$ and $p^\ast_\rb$ and (b) conditional distributions $p_{\ra \vert \rb}$ and $p_{\rb \vert \ra}$.
The unconditional models can be used for density estimation and sampling from $\ra$ and $\rb$ whereas the conditional models can be used for translating (i.e., conditional sampling) from $\rb \to \ra$ and $\ra \to \rb$.

Before presenting the \name{} framework, we note two observations.
For task (a), we need datasets from the domains $\ra$ and $\rb$ respectively for learning.
For task (b), we note that the problem is \textit{underconstrained} since we are only given data from the marginal distributions and hence, it is unclear how to learn the conditional distribution that relates the datapoints from the two domains.
Hence, we need additional inductive biases on our learning algorithm that can learn useful conditional distributions,
In practice, many such forms of inductive biases have been designed and shown to be useful across relavant tasks such as cross-domain translation and domain adaptation~\citep{zhu2017toward,liu2017unsupervised}.

\subsection{Representation}
We will use a graphical model to represent the relationships between the domains.
Consider a Bayesian network $\ra \leftarrow \rz \rightarrow \rb$ with two sets of observed random variables (domains) $\ra\subseteq \mathbb{R}^n$ and $\rb\subseteq \mathbb{R}^n$ and a parent set of latent random variables $\rz \subseteq \calZ$.

The latent variables $\rz$ indicate a shared feature space between the observed variables $\ra$ and $\rb$, which will be exploited later for efficient learning and inference.
While $\rz$ is unobserved, we assume a prior density $p_\rz$ over these variables, such as an isotropic Gaussian. 
Finally, to compactly specify the joint distribution over all sets of variables, we constrain the relationship between $\ra$ and $\rz$, and $\rb$ and $\rz$ to be invertible. That is, we specify mappings $\gza$ and $\gzb$ such that the respective inverses $\gaz=\gza^{-1}$ and $\gbz=\gzb^{-1}$ exist. 
Notice that such a representation naturally provides a mechanism to  translate from one domain to another as the composition of two invertible mappings:
\begin{align}
    \gab &= \gzb \circ \gaz\label{eq:invert_1}\\
    \gba &= \gza \circ \gbz.\label{eq:invert_2}
\end{align}
Since composition of invertible mappings is invertible, both $\gab$ and $\gba$ are invertible. In fact, it is straightforward to observe that $\gab$ and $\gba$ are inverses of each other:
\begin{align}
    \gab^{-1} &= (\gzb \circ \gaz)^{-1} = \gaz^{-1} \circ \gzb^{-1} \nonumber \\
    &= \gza \circ \gbz = \gba.\label{eq:invert_3}
\end{align}

\subsection{Learning Algorithms \& Objectives}

As discussed in the preliminaries, each of the individual flows $\gza$ and $\gzb$ express a model with density $p_\ra$ and $p_\rb$ respectively and can be trained independently via maximum likelihood estimation, adversarial learning, or a hybrid objective.
However, our goal is to perform sample-efficient learning by exploiting data from other domains as well as learn a conditional mapping across the two domains.
For both these goals, we require learning algorithms which use data from both domains for parameter estimation.
 Unless mentioned otherwise, all our results that hold for a particular domain $\ra$ will have a natural counterpart for the domain $\rb$.
 
\subsubsection{Adversarial Training}
Instead of adversarial training of $\gza$ and $\gzb$ independently, we can directly perform adversarial training of the mapping $\gba$.
 That is, we first generate data from $\gba$ using the prior density given as $p^\ast_\rb$.
We also introduce a critic $C_\ra$ which distinguishes real samples $a \sim p^\ast_\ra$ with the generated samples $\gba(b)$ for $b \sim p^\ast_\rb$.
For example, the cross-entropy GAN loss in this case is given as:
\begin{align}\label{eq:adv_2}
  \mathcal{L}_{\textnormal{GAN}}(C_\ra, \gba)&= \E_{a \sim p^\ast_\ra}[\log C_\ra(a)] \nonumber\\
  &+ \E_{b \sim p^\ast_\rb}[\log(1- C_\ra(\gba(b)))].
\end{align}
The expectations above are approximated empirically via datasets $\calD_\ra$ and $\calD_\rb$ respectively.

\subsubsection{Maximum Likelihood Estimation}
Unlike adversarial training, flow models trained with maximum likelihood estimation (MLE) explicitly require a prior $p_\rz$ with a tractable density (e.g., isotropic Gaussian) to evaluate model likelihoods using the change-of-variables formula in Eq.~\ref{eq:cov}.
Due to this tractability requirement, we cannot substitute $p_\rz$ with the unknown $p^\ast_\rb$ for MLE.
Instead, we can share parameters between the two mappings.
The extent of parameter sharing depends on the similarity across the two domains; for highly similar domains, entire architectures could potentially be shared in which case $\gza=\gzb$.

\subsubsection{Hybrid Training}
Both MLE and adversarial training objectives can be combined into a single training objective.
The most general \name{} objective is given as:
\begin{align}\label{eq:\name{}}
    &\mathcal{L}_{\textnormal{\name{}}}(\gba, C_\ra, C_\rb; \lambda_\ra, \lambda_\rb) \nonumber\\
    &
    =\mathcal{L}_{\textnormal{GAN}}(C_\ra, \gba) + \mathcal{L}_{\textnormal{GAN}}(C_\rb, \gab) 
    \nonumber \\
    &
    - \lambda_\ra \mathcal{L}_{\textnormal{MLE}}(\gza) - \lambda_\rb \mathcal{L}_{\textnormal{MLE}}(\gzb) 
\end{align}

where $\lambda_\ra \geq 0$ and $\lambda_\rb\geq 0$ are hyperparameters that control the strength of the MLE terms for domains $\ra$ and $\rb$ respectively. The \name{}  objective is minimized w.r.t. the parameters of the generator $\gab$ and maximized w.r.t. parameters of the critics $C_\ra$ and $C_\rb$. Notice that $\mathcal{L}_{\textnormal{\name{}}}$ is expressed as a function of the critics $C_\ra, C_\rb$ and only $\gba$ since the latter also encompasses the other parametric functions appearing in the objective ($\gab,\gza,\gzb$) via the invertibility constraints in Eqs.~\ref{eq:invert_1}-\ref{eq:invert_3}. When $\lambda_\ra=\lambda_\rb=0$, we perform pure \textit{adverarial training} and the prior over $\rz$ plays no role in learning.
On the other hand, when $\lambda_\ra=\lambda_\rb\to\infty$, 
    we can perform pure \textit{MLE training} to learn the invertible generator. Here, the critics $C_\ra, C_\rb$ play no role since the adversarial training terms are ignored. 

\subsection{Inference}
\name{} can be used for both conditional and unconditional sampling at test time. For conditional sampling as in the case of domain translation, we are given a datapoint $b \in \rb$ and we can draw the corresponding cross-domain translation in domain $\ra$ via the mapping $\gba$. 
For unconditional sampling, we require $\lambda_\ra\neq 0$  since doing so will activate the use of the prior $p_\rz$ via the MLE terms in the learning objective.
Thereafter, we can obtain samples by first drawing $z\sim p_\rz$ and then applying the mapping $\gza$ to $z$.
Furthermore, the same $z$ can be mapped to domain $\rb$ via $\gzb$.
Hence, we can sample paired data $(\gza(z), \gzb(z))$ given $z \sim p_\rz$.

\section{Theoretical Analysis}
The \name{} objective consists of three parametric models: one generator $\gba\in\calG$, and two critics $C_\ra\in \calC_\ra, C_\rb \in \calC_\rb$. 
Here, $\calG, \calC_\ra, \calC_\rb$ denote model families specified e.g., via deep neural network based architectures.
In this section, we analyze the optimal solutions to these parameterized models within well-specified model families.

\subsection{Optimal Generators}

Our first result characterizes the conditions under which the optimal generators exhibit \textit{marginal-consistency} for the data distributions defined over the domains $\ra$ and $\rb$.

\begin{definition}\label{def:mconst} (Marginal-consistency)
Let $p_{\rx, \ry}$ denote the joint distribution between two domains $\calX$ and $\calY$.  
An invertible mapping $\gyx: \calY \to \calX$ is marginally-consistent w.r.t. two arbitrary distributions ($p_{\rx},p_{\ry}$) iff for all $x\in \calX$, $y \in \calY$:
\begin{align}\label{eq:marginally_consistent}
    p_{\rx} (x)=  \bigg\{\begin{array}{lr} p_\ry(y) \ildj{{\gyx}}{\rx}{x}, & \text{if } x=\gyx(y)\\
        0, & \text{otherwise}.
        \end{array}
\end{align}
\end{definition}

Next, we show that \name{} is marginally-consistent for well-specified model families.

\begin{lemma}\label{thm:mle_implies_adv_consistency}
Let $\calG_{\ra}$ and $\calG_{\rb}$ denote the class of invertible mappings represented by the \name{} architecture for mapping $\rz \to \ra$ and $\rz \to \rb$. 
For a given choice of prior distribution $p_\rz$, if there exist mappings $\gza^\ast \in \calG_\ra, \gzb^\ast \in \calG_\rb$ that are marginally-consistent w.r.t. ($p^\ast_\ra,p_\rz$) and ($p^\ast_\rb,p_\rz$) respectively, then the mapping $\gba^\ast= \gza^\ast \circ \gzb^{\ast^{-1}}$ is marginally-consistent w.r.t. ($p^\ast_\ra,p^\ast_\rb$). 
\end{lemma}
The result follows directly from Definition~\ref{def:mconst} and change-of-variables applied to the mapping $\gba^\ast$.

\begin{theorem}\label{thm:full_consistency}
Assume that the model families for the critics $C_\ra: \ra \to [0,1]$ and $C_\rb: \rb \to [0,1]$ are the set of all measurable functions for the cross-entropy GAN objective. 
Then,  $\gba^\ast$ (as defined in Lemma~\ref{thm:mle_implies_adv_consistency}) globally minimizes the \name{} objective in Eq.~\ref{eq:\name{}} for any $\lambda_{\ra} \geq 0, \lambda_{\rb} \geq 0$.
\end{theorem}

\textit{Proof.} See Appendix A.1. Theorem~\ref{thm:full_consistency} suggests that optimizing the \name{} objective will recover the marginal data distributions $p^\ast_\ra$ and $p^\ast_\rb$ under suitable conditions.
For the other goal of learning cross-domain mappings, we note that marginally-consistent mappings w.r.t. a target data distribution (such as $p^\ast_\ra$) and a target prior density (such as $p^\ast_\rb$) need not be unique. 
While a cycle-consistent, invertible model family mitigates the underconstrained nature of the cross-domain translation problem, it does not provably eliminate it.
We provide some non-identifiable constructions in Appendix A.3 and leave open the exploration of additional constraints that guarantee identifiability for future exploration.

\subsection{Optimal Critics}

Unlike standard adversarial training of an unconditional normalizing flow model~\citep{flow-gan,danihelka2017comparison}, the \name{} model involves two critics.
Here, we are interested in characterizing the dependence of the optimal critics for a given invertible mapping $\gab$.
Consider the \name{} framework where the GAN loss terms in Eq.~\ref{eq:\name{}} are specified via the cross-entropy objective in Eq.~\ref{eq:adv}. 
For this model, we can relate the optimal critics using the following result.

\begin{theorem}\label{thm:opt_critics}
Let $p^\ast_\ra$ and $p^\ast_\rb$ denote the true data densities for domains $\ra$ and $\rb$ respectively. Let $C^\ast_\ra$ and $C^\ast_\rb$ denote the optimal critics for the \name{} objective with the cross-entropy GAN loss for any fixed choice of the invertible mapping $\gab$. Letting $b=\gab(a)$ for any $a \in \ra$, we have:
\begin{align}
    C^\ast_\ra(a) &= \frac{C^\ast_\rb(b)p^\ast_\ra(a)}{p^\ast_\ra(a) + p^\ast_\rb(b) (1-C^\ast_\rb(b)) \ildj{\gba}{\ra}{a}}.
\end{align}
\end{theorem}

\textit{Proof.} See Appendix A.2. In essence, the above result shows that the optimal critic for one domain, w.l.o.g. say $\ra$, can be directly obtained via the optimal critic of another domain $\rb$ for any choice of the invertible mapping $\gab$, assuming access to the data marginals $p^\ast_\ra$ and $p^\ast_\rb$.

\subsection{Exact Cycle Consistency}

\begin{figure}[t]
\centering

\begin{subfigure}[b]{\columnwidth}
\centering
 \begin{tikzpicture}
 \node[circle] (a) at (-2,1) [draw, minimum width=0.5cm,minimum height=0.5cm] {$\ra$};
 \node[circle] (b) at (2,1) [draw, minimum width=0.5cm,minimum height=0.5cm] {$\rb$};
\node[circle] (ca) at (-2,-1) [draw, minimum width=0.5cm,minimum height=0.5cm] {$\ry_\ra$};
\node[circle] (cb) at (2,-1) [draw, minimum width=0.5cm,minimum height=0.5cm] {$\ry_\rb$};
 \foreach \from/\to in {a/b}
\draw [->] (a) to [out=30,in=150] node [above] {$\gab$} (b);
\draw [->] (b) to [out=210,in=330] node [below] {$\gba$} (a);
\draw [->] (a) -- node [left] {$C_\ra$} (ca);
\draw [->] (b) -- node [right] {$C_\rb$} (cb);
  \end{tikzpicture}
  \caption{CycleGAN}
\end{subfigure}
~~~~~~~~~~~~~~~~~~~~~~~~~~~~~~~~~~~~~~~~~~~~~~~~~~~~~~~~~~~~~~~~~~~~~~~~~~~~~~~~~~~~~~~~~~~~~~~~~~~~
\begin{subfigure}[b]{\columnwidth}
\centering
 \begin{tikzpicture}
 \node[circle] (a) at (-2,1) [draw, minimum width=0.5cm,minimum height=0.5cm] {$\ra$};
 \node[circle] (b) at (2,1) [draw, minimum width=0.5cm,minimum height=0.5cm] {$\rb$};
  \node[circle] (z) at (0,2.5) [draw, minimum width=0.5cm,minimum height=0.5cm] {$\rz$};
  \node[circle] (ca) at (-2,-1) [draw, minimum width=0.5cm,minimum height=0.5cm] {$\ry_\ra$};
\node[circle] (cb) at (2,-1) [draw, minimum width=0.5cm,minimum height=0.5cm] {$\ry_\rb$};
 \foreach \from/\to in {z/a, b/z}
\draw [->] (\from) -- (\to);
\draw [->] (a) -- node [above left] {$\gaz=\gza^{-1}$}(z);
\draw [->] (z) -- node [above right] {$\gbz=\gzb^{-1}$}(b);
\draw [->] (a) -- node [left] {$C_\ra$} (ca);
\draw [->] (b) -- node [right] {$C_\rb$} (cb);
  \end{tikzpicture}
  \caption{\name{}}
\end{subfigure}
\caption{CycleGAN v.s. \name{} for unpaired cross-domain translation. Unlike CycleGAN, \name{} specifies a single invertible mapping $\gaz \circ \gbz^{-1}$ that is exactly cycle-consistent, represents a shared latent space $\rz$ between the two domains, and can be trained via both adversarial training and exact maximum likelihood estimation. Double-headed arrows denote invertible mappings. $\ry_\ra$ and $\ry_\rb$ are random variables denoting the output of the critics used for adversarial training. 
}
\label{fig:cyclegan_vs_alignflow}
 \end{figure}
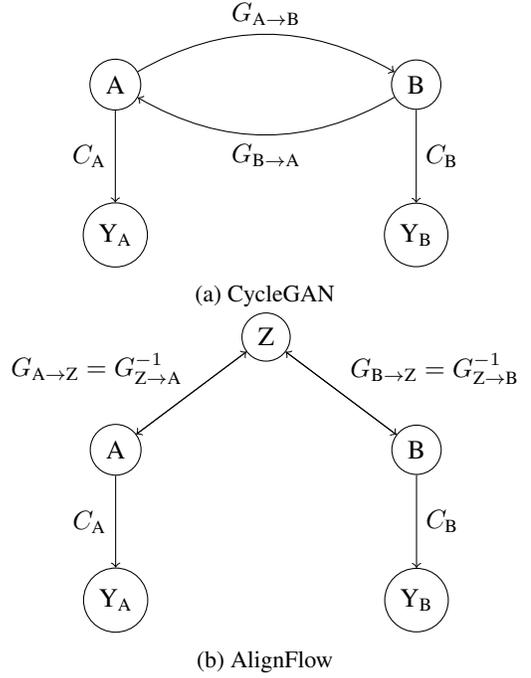
 
 So far, we have only discussed objectives that are marginally-consistent with respect to data distributions $p^\ast_\ra$ and $p^\ast_\rb$.
However, many domain alignment tasks such as cross-domain translation require can be cast as learning a joint distribution $p^\ast_{\ra,\rb}$. 
As discussed previously, this problem is underconstrained given unpaired datasets $\calD_\ra$ and $\calD_\rb$ and the learned marginal densities alone do not guarantee learning a mapping that is useful for downstream tasks. 
Cycle consistency, as proposed in CycleGAN~\citep{cycle-gan}, is a highly effective learning objective that encourages learning of meaningful cross-domain mappings such that the data translated from domain $\ra$ to $\rb$ via $\gab$ to be mapped back to the original datapoints in $\ra$ via $\gba$. That is, $\gba(\gab(a)) \approx a$ for all $a \in \ra$. 
Formally, the cycle-consistency loss for translation from $\ra$ to $\rb$ and back is defined as:
\begin{align}\label{eq:cyc_cons}
    \mathcal{L}_{\textnormal{Cycle}}(\gba, \gab)&=  E_{a \sim p^\ast_\ra}[ \Vert\gba(\gab(a)) - a  \Vert_1 ].
\end{align}
 Symmetrically, we have a cycle consistency term $\mathcal{L}_{\textnormal{Cycle}}(\gab, \gba)$ in the reverse direction that encourages $\gab(\gba(b)) \approx b$ for all $b \in \rb$. 
 Next, we show that \name{} is \textit{exactly} cycle consistent.

\begin{proposition}
 Let $\calG$ denote the class of invertible mappings represented by an arbitrary \name{} architecture.
 For any $\gba \in \calG$, we have:
 \begin{align}
     \mathcal{L}_{\textnormal{Cycle}}(\gba, \gab) = 0\\
     \mathcal{L}_{\textnormal{Cycle}}(\gab, \gba) = 0
 \end{align}
 where $\gab=\gba^{-1}$ by design.
\end{proposition}
The proposition follows directly from the invertible design of the \name{} framework (Eq.~\ref{eq:invert_3}) and Eq.~\ref{eq:cyc_cons}.

\subsubsection{Comparison with CycleGAN}
We illustrate and compare \name{} and CycleGAN in Figure~\ref{fig:cyclegan_vs_alignflow}.
CycleGAN parameterizes two independent cross-domain mappings $\gab$ and $\gba$, whereas \name{} only specifies a single, invertible mapping.
Learning in a CycleGAN is restricted to an adversarial training objective along with additional cycle-consistent loss terms.
In contrast, \name{} is exactly consistent and can be trained via adversarial learning, MLE, or a hybrid (Eq.~\ref{eq:\name{}}) without the need for additional loss terms to enforce cycle consistency.
Finally, inference in CycleGAN is restricted to conditional sampling since it does not involve any latent variables $\rz$ with easy-to-sample prior densities.
As described previously, \name{} permits both conditional and unconditional sampling.

\subsubsection{Comparison with UNIT and CoGAN} Models such as CoGAN~\citep{liu2016coupled} and its extension UNIT~\citep{liu2017unsupervised} also consider adding a shared-space constraint between two different domain decoders.
These models again can only enforce approximate cycle consistency and introduce additional encoder networks.
Moreover, they only approximate lower bounds to the log-likelihood unlike \name{} which permits exact MLE training.

\begin{figure*}[ht]
\begin{center}
\centerline{\includegraphics[width=0.9\textwidth]{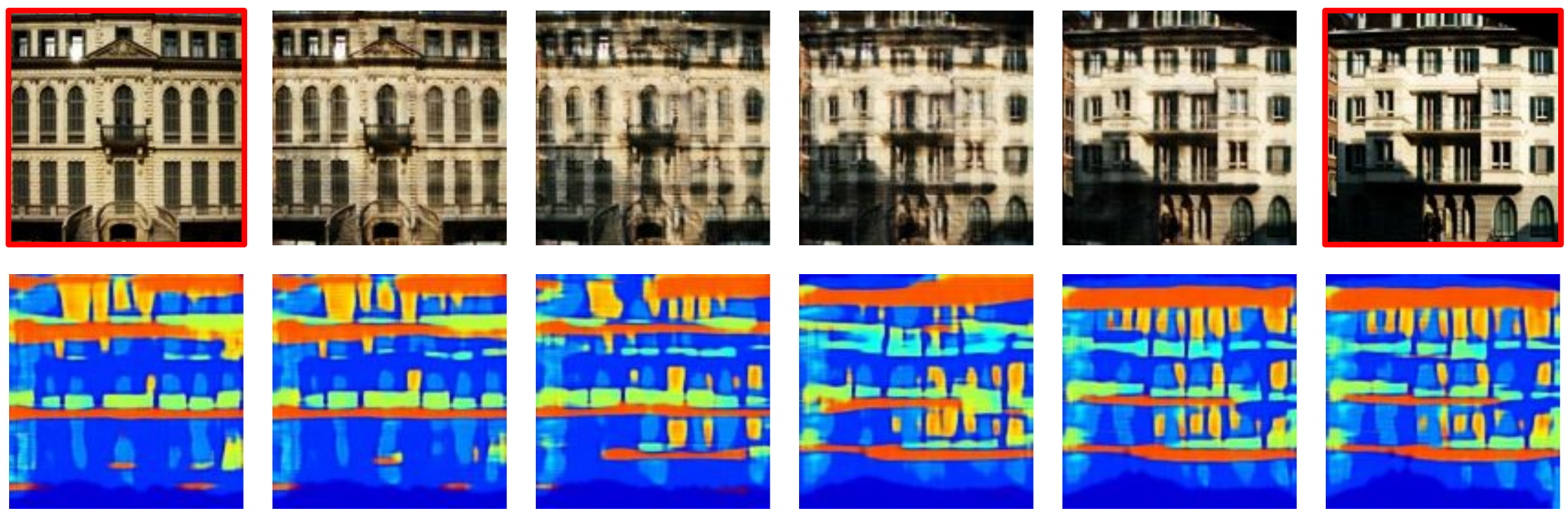}}
\caption{Latent space interpolation on Facades. 
\textbf{Top:} Left- and right-most images are sampled from $\calD_\ra$ (red boxes).
Interpolation is then performed in latent space and then decoded using $\gza$.
We see semantically meaningful changes across the row, \textit{e.g.,} in the shadow and the style of entrance to the building.
\textbf{Bottom:} For each image in the top row, its latent representation is decoded into the target domain using $\gzb$. 
Inspection of the orange regions indicates a change from 3 floors (left) to 4 floors (right).}
\label{fig:facades-interpolation}
\end{center}
\end{figure*}
\begin{table}[t]
    \centering
    \caption{Mean Squared Error (MSE) comparing CycleGAN and variants of \name{}(AF) on paired test sets. MSE is computed pixelwise after normalizing images to $(-1, 1)$.}
    \label{table:mse-results}
    \begin{tabular}{@{}l|lcc@{}}
    \toprule
   Dataset                  &   Model  & MSE ($\ra \to \rb$)      & MSE ($\rb \to \ra$)      
   \\ \midrule
     \parbox[t]{2mm}{\multirow{3}{*}{\rotatebox[origin=c]{90}{Facades}}} 
       & CycleGAN                        & 0.7129          & 0.3286          
       \\
    &AF (ADV only)         & 0.6727	        & 0.2679         
    \\
    &AF (Hybrid)              & \textbf{0.5801} & \textbf{0.2512} 
    \\
    &AF (MLE only)    & 0.9014          & 0.5960          
    \\ \midrule
    \parbox[t]{2mm}{\multirow{3}{*}{\rotatebox[origin=c]{90}{Maps}}}  & CycleGAN                           & 0.0245          & 0.0953          
    \\
    &AF (ADV only)           & 0.0385	        & 0.1123         
    \\
    &AF (Hybrid)                & \textbf{0.0209} & \textbf{0.0897} 
    \\
    &AF (MLE only)       & 0.0452          & 0.1746          
    \\ \midrule

    \parbox[t]{2mm}{\multirow{3}{*}{\rotatebox[origin=c]{90}{CityScapes}}}  &CycleGAN                      & 0.1252          & \textbf{0.1200}         
    \\
    &AF (ADV only)     & 0.2569          & 0.2196         
    \\
    &AF (Hybrid)         & \textbf{0.1130}          & 0.1462         
    \\
    &AF (MLE only) & 0.2526          & 0.2272         
    \\
    \bottomrule
    \end{tabular}
\end{table}
\section{Experimental Evaluation}

To achieve our two goals of data-efficient modeling of individual domains and effective cross-domain mappings, we evaluate \name{} on two tasks: 
(a) unsupervised image-to-image translation,
 and (b) unsupervised domain adaptation.
For additional experimental details, results, and analysis beyond those stated below, we refer the reader to Appendix~B.

\subsection{Image-To-Image Translation}

We evaluate \name{} on three image-to-image translation datasets used by \citet{cycle-gan}: Facades, Maps, and CityScapes~\citep{cityscapes}.
These datasets are chosen because they provide one-to-one aligned image pairs, so one can quantitatively evaluate unpaired image-to-image translation models via a distance metric such as mean squared error (MSE) between generated examples and the corresponding ground truth. 
While MSE can be substituted for perceptual losses in other scenarios, it is a suitable metric for evaluating datasets with one-to-one ground pairings.
Note that the task at hand is unpaired translation and hence, the pairing information is omitted during training and only used for evaluation.  

We report the MSE for translations on the test sets after cross-validation of hyperparameters in Table \ref{table:mse-results}.
For hybrid models, we set $\lambda_\ra=\lambda_\rb$ and report results for the best values of these hyperparameters. 
We observe that while learning \name{} via adversarial training or MLE alone is not as competitive as CycleGAN, hybrid training of \name{} significantly outperforms CycleGAN in almost all cases.
Specifically, we observe that MLE alone typically performs worse than adversarial training, but together both these objectives seem to have a regularizing effect on each other.
Qualitative interpolations on the Facades dataset are shown in Figure~\ref{fig:facades-interpolation}.

\subsection{Unsupervised Domain Adaptation}

\begin{table*}[t]
    \centering
    \caption{Test classification accuracies for domain adaptation from source $\to$ target. The \textbf{source only} and \textbf{target only} models directly use classifiers trained on the source and target datasets respectively. Baseline numbers taken from the cited works.}
    \label{table:domain-adaptation}
    \begin{tabular}{@{}lcccc@{}}
     \toprule
     Model & MNIST $\to$ USPS & USPS $\to$ MNIST & SVHN $\to$ MNIST \\
     \midrule
     source only & 82.2 $\pm$ 0.8 & 69.6 $\pm$ 3.8 & 67.1 $\pm$ 0.6 \\
     ADDA~\citep{adda} & 89.4 $\pm$ 0.2 & 90.1 $\pm$ 0.8 & 76.0 $\pm$ 1.8 \\
     CyCADA~\citep{cycada} & 95.6 $\pm$ 0.2 & 96.5 $\pm$ 0.1 & 90.4 $\pm$ 0.4\\
     UNIT~\citep{liu2017unsupervised} & 95.97 & 93.58 & 90.53\\ 
      \name{} & \textbf{96.2 $\pm$ 0.2} & \textbf{96.7 $\pm$ 0.1} & \textbf{91.0 $\pm$ 0.3} \\ \midrule
     target only & 96.3 $\pm$ 0.1 & 99.2 $\pm$ 0.1 & 99.2 $\pm$ 0.1  \\
    \bottomrule
    \end{tabular}
\end{table*}

\begin{figure}[t]
    \centering
    \begin{subfigure}[b]{0.45\columnwidth}
    \centering
    \includegraphics[width=\columnwidth]{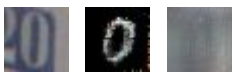}
    \caption{}
    \end{subfigure}
    \begin{subfigure}[b]{0.45\columnwidth}
    \centering
    \includegraphics[width=\columnwidth]{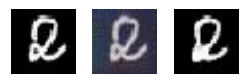}
    \caption{}
    \end{subfigure}
    \begin{subfigure}[b]{0.45\columnwidth}
    \centering
    \includegraphics[width=\columnwidth]{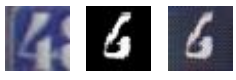}
    \caption{}
    \end{subfigure}
    \begin{subfigure}[b]{0.45\columnwidth}
    \centering
    \includegraphics[width=\columnwidth]{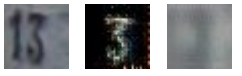}
    \caption{} 
    \end{subfigure}
    \caption{Examples of failure modes for CycleGAN reconstructions in SVHN $\leftrightarrow$ MNIST cross-domain translation. In each group of three images in (a-d), a real example from the source domain (SVHN) is shown on the left, the translated image in the target domain (MNIST) is shown at center, and the reconstructed image in the source domain based on the translation is shown on the right.}
    \label{fig:cycle-failures}
\end{figure}
\begin{figure*}[t]
    \centering
    \begin{subfigure}[b]{\columnwidth}
    \centering
    \includegraphics{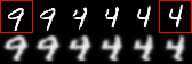}
    \caption{MNIST$\to$USPS}
    \end{subfigure}
    ~~~~  ~~~~  
    \begin{subfigure}[b]{\columnwidth}
    \centering
    \includegraphics{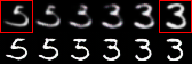}
    \caption{USPS$\to$MNIST}
    \end{subfigure}
    
    \caption{Multi-domain latent space interpolations. 
    \textbf{Top:} Left-most and right-most images are sampled from $\calD_\ra$ (in red boxes).
Interpolation is then performed in latent space and then decoded using $\gza$.
\textbf{Bottom:} For each corresponding image in the top row, its latent representation is decoded into the target domain using $\gzb$.
    Note how both class identity and style are preserved in the interpolated pairs of digits in the two domains.
    Also, notice that the USPS images (even the true ones in red boxes) are slightly blurred due to the upscaling applied as standard preprocessing.
    }
    \label{fig:mnist-interpolation}
\end{figure*}

In unsupervised domain adaptation~\citep{saenko2010adapting}, we are given data from two related domains: a source and a target domain.
For the source, we have access to both the input datapoints and their labels. For the target, we are only provided with input datapoints without any labels. 
Using the available data, the goal is to learn a classifier for the target domain.
We extend \citet{cycada} to use an \name{} architecture and objective (adversarially trained Real-NVPs~\citep{real-nvp} here) in place of CycleGAN for this task. 

A variety of algorithms have been proposed for the above task which seek to match pixel-level or feature-level distributions across the two domains.
See Appendix B for more details.
For fair and relevant comparison, we compare against baselines Cycada~\citep{cycada} and UNIT~\citep{liu2017unsupervised} which involve pixel-level translations and are closest to the current work.
We evaluate across all pairs of source and target datasets as in \citet{cycada} and \citet{liu2017unsupervised}: MNIST, USPS, SVHN, which are all image datasets of handwritten digits with 10 classes.
In Table~\ref{table:domain-adaptation}, we see that \name{} outperforms both Cycada~\cite{cycada} (based on CycleGAN) and UNIT~\citep{liu2017unsupervised} in all cases.
Combining \name{} with other state-of-the-art 
adaptation approaches e.g.,~\citet{shu2018dirt},~\citet{long2018conditional},~\citet{kumar2018co},~\citet{liu2018unified},~\citet{sankaranarayanan2018generate},~\citet{liu2018detach} is an interesting direction for future work.

In Figure~\ref{fig:cycle-failures}, we show some failure modes of using approximately cycle-consistent objectives for the Cycada model. 
Notice that the image label and style changes or becomes unrecognizable in translating and reconstructing the input.
In contrast, \name{} is exactly cycle consistent and hence, the source reconstructions based on the translated images will be exactly the source image by design.

\subsection{Multi-Domain Concurrent Interpolations}

The use of a shared latent space in \name{} allows us to perform paired interpolations in two domains simultaneously.
While pure MLE without any parameter sharing does not give good alignment, pure adversarial training cannot be used for unconditional sampling since the prior $p_\rz$ is inactive.
Hence, we use \name{} models trained via a hybrid objective for latent space interpolations.
In particular, we sample two datapoints
$a', a''\in\calD_{\ra}$ and obtain their latent representations $z', z''\in\mathcal{Z}$ via $\gza$. 
Following~\citet{real-nvp}, we compute interpolations in the polar space as $\tilde{z} = z'\sin\phi + z''\cos\phi$ for several values of $\phi\in (0, 2\pi)$. 
Finally, we map $\tilde{z}$ to either back to domain $\ra$ via $\gza$ and $\rb$ via $\gzb$. 
We show this empirically on the MNIST/USPS datasets in Figure~\ref{fig:mnist-interpolation}.
We see that many aspects of style and content are preserved in the interpolated samples.

\section{Related Work}\label{sec:related}
A key assumption in unsupervised domain alignment is the existence of a deterministic or stochastic mapping $\gab$ such that the distribution of $\rb$ matches that of $\gab(\ra)$, and vice versa. This assumption can be incorporated as a marginal distribution-matching constraint into the objective using an adversarially-trained GAN critic~\citep{gan}. However, this objective is under-constrained. 
To partially mitigate this issue, CycleGAN~\citep{cycle-gan}, DiscoGAN~\citep{kim2017learning}, and DualGAN~\citep{yi2017dualgan} added an approximate cycle-consistency constraint that encourages $\gba \circ \gab$ and $\gab \circ \gba$ to behave like identity functions on domains $\ra$ and $\rb$ respectively. 
While cycle-consistency is empirically very effective, alternatives based on variational autoencoders that do not require either cycles or adversarial training have also been proposed recently~\citep{hoshen2018non,hoshen2018nam}.

Models such as CoGAN~\citep{liu2016coupled}, UNIT~\citep{liu2017unsupervised}, and CycleGAN~\citep{cycle-gan} have since been extended to enable one-to-many mappings~\citep{huang2018multimodal,zhu2017toward} as well as multi-domain alignment~\citep{choi2018stargan}. Our work focuses on the one-to-one unsupervised domain alignment setting. In contrast to previous models, \name{} leverages both a shared latent space and \emph{exact} cycle-consistency. To our knowledge, \name{} provides the first demonstration that invertible models can be used successfully in lieu of the cycle-consistency objective. Furthermore, \name{} allows the incorporation of exact maximum likelihood training, which we demonstrated to induce a meaningful shared latent space that is amenable to interpolation.

\section{Conclusion \& Future Work}
We presented \name{}, a generative framework for learning from multiple data sources based on normalizing flow models.
\name{} has several attractive properties: it guarantees exact cycle-consistency via a single cross-domain mapping, learns a shared latent space across two domains, and permits a flexible training objective which can combine adversarial training and exact maximum likelihood estimation.
Theoretically, we derived conditions under which the \name{} model learns marginals that are consistent with the underlying data distributions.
Finally, our empirical evaluation demonstrated significant gains on unsupervised domain translation and adaptation, and an increase in inference capabilities, e.g., paired interpolations in the latent space for two domains.

In the future, we plan to consider extensions of \name{} for learning stochastic, multimodal mappings~\citep{zhu2017toward} and translations across more than two domains~\citep{choi2018stargan}.
Exploring recent advancements in invertible architectures e.g.,~\citet{ho2019flow++},~\citet{huang2018neural},~\citet{chen2018neural},~\citet{grathwohl2018ffjord} within \name{} is another natural promising direction for future work.
With a handle on model likelihoods and invertible inference, we are optimistic that \name{} can aid the characterization of useful structure that guarantees identifiability in underconstrained problems such as domain alignment~\citep{cui2014generalized,galanti2017role,alvarez2018towards,wu2019domain}.

\subsection*{Acknowledgements}
We are thankful to Kristy Choi, Daniel Levy, and Kelly Shen for helpful feedback.
This research was supported by Samsung, FLI, AWS.
AG is additionally supported by a Lieberman Fellowship and a Stanford Data Science Scholarship.
\fontsize{9pt}{10pt} \selectfont
\bibliography{refs}
\bibliographystyle{aaai}
\clearpage
\input{appendix.tex}

\end{document}

%% file: appendix.tex

\section*{Appendices}
\begin{appendix}

\section{Proofs of Theoretical Results}
\subsection{Proof of Theorem~\ref{thm:full_consistency}}\label{app:proof_full_consistency}

\begin{proof}

Since the maximum likelihood estimate minimizes the KL divergence between the data and model distributions, the optimal value for $\mathcal{L}_{\textnormal{MLE}}(\gza)$ is attained at a marginally-consistent mapping, say $\gza^\ast$.
Symmetrically, there exists a marginally-consistent mapping $\gzb^\ast$ that optimizes $\mathcal{L}_{\textnormal{MLE}}(\gzb)$.

From Theorem 1 of \citet{gan}, we know that the cross-entropy GAN objective $\mathcal{L}_{\textnormal{GAN}}(C_\ra, \gba)$ is globally minimized when $p_{\ra}=p^\ast_\ra$ and critic is Bayes optimal.
Further, from Lemma~\ref{thm:mle_implies_adv_consistency}, we know that  $\gba^\ast$ is marginally-consistent w.r.t. ($p^\ast_\ra,p^\ast_\rb$).
Hence, $\gba^\ast$ globally minimizes $\mathcal{L}_{\textnormal{GAN}}(C_\ra, \gba)$. 
Symmetrically, $\gab^\ast=\gba^{\ast^{-1}}$ globally minimizes $\mathcal{L}_{\textnormal{GAN}}(C_\rb, \gab)$.

Since $\gba^\ast=\gza^\ast \circ {\gzb^{\ast^{-1}}}$ globally optimizes all the individual loss terms in the \name{} objective in Eq.~\ref{eq:\name{}}, it globally optimizes the overall objective for any value of $\lambda_{\ra} \geq 0, \lambda_{\rb} \geq 0$.

\end{proof}

\subsection{Proof of Theorem~\ref{thm:opt_critics}}\label{app:proof_opt_critics}
\begin{proof}
First, we note that only the GAN loss terms depend on $C_\ra$ and $C_\rb$. Hence, the MLE terms are constants for a fixed $\gba$ and hence, can be ignored for deriving the optimal critics.
Next, for any GAN trained with the cross-entropy loss as specified in Eq~\ref{eq:adv_2}, we know that the Bayes optimal critic $C^\ast_\ra$ prediction for any $a\in \ra$ is given as:
\begin{align}\label{eq:opt_ca}
    C^\ast_\ra(a) &= \frac{p^\ast_\ra(a)}{p^\ast_\ra(a) + p_\ra(a)}
\end{align}
See Proposition 1 in \citet{gan} for a proof. 

We can relate the densities $p_\ra(a)$ and $p_\rb(b)$ via the change of variables as:
\begin{align}\label{eq:pab}
    p_\ra(a) &= p_\rb(b) \ildj{\gba}{\ra}{a}
\end{align}
where $b=\gab(a)$.

Substituting the expression for density of $p_\ra(a)$ from Eq.~\ref{eq:pab} in  Eq.~\ref{eq:opt_ca}, we get:
\begin{align}\label{eq:opt_ca_2}
    C^\ast_\ra(a) &= \frac{p^\ast_\ra(a)}{p^\ast_\ra(a) + p_\rb(b) \ildj{\gba}{\ra}{a}}
\end{align}
where $b=\gab(a)$.

Symmetrically, using Proposition 1 in \citet{gan} we have the Bayes optimal critic $C^\ast_\rb$ for any $b\in \rb$  given as:
\begin{align}\label{eq:opt_cb}
    C^\ast_\rb(b) &= \frac{p^\ast_\rb(b)}{p^\ast_\rb(b) + p_\rb(b)}.
\end{align}

Rearranging terms in Eq.~\ref{eq:opt_cb}, we have:
\begin{align}\label{eq:opt_cb_2}
    p_\rb(b) = p^\ast_\rb(b) \left(\frac{1}{C^\ast_\rb(b)}-1\right)
\end{align}
for any $b \in \rb$.

Substituting the expression for density of $p_\rb(b)$ from Eq.~\ref{eq:opt_cb_2} in Eq.~\ref{eq:opt_ca_2}, we get:
\begin{align}\label{eq:opt_ca_3}
    C^\ast_\ra(a) &= \frac{C^\ast_\rb(b)p^\ast_\ra(a)}{p^\ast_\ra(a) + p^\ast_\rb(b) (1-C^\ast_\rb(b)) \ildj{\gba}{\ra}{a}} 
\end{align}
where $b=\gab(a)$.

\end{proof}

\subsection{Non-Identifiability of Cross-Domain Mappings}\label{app:mle_non_id}

As discussed, marginal consistency along with invertibility can only reduce the underconstrained nature of the unpaired cross-domain translation problem, but not completely eliminate it.
In the following result, we identify one such class of non-identifiable model families for the MLE-only objective of \name{} ($\lambda_\ra=\infty, \lambda_\rb=\infty$).
We will need the following definitions.
\begin{definition}
Let $\calS_n$ denotes the symmetric group on $n$ dimensional permutation matrices.
A function class for the cross-domain mappings $\calG$ is closed under permutations iff for all $\gba \in \calG$, $S \in \calS_n$, we have $\gba \circ S \in \calG$.
\end{definition}
\begin{definition}
A density $p_\rx$ is symmetric iff for all $x \in \calX \subseteq \mathbb{R}^n, S \in \calS_n$, we have $p_\rx(x) = p_\rx(Sx)$.
\end{definition}
Examples of distributions with symmetric densities include the isotropic Gaussian and Laplacian distributions.




\begin{proposition}\label{thm:mle_only}
Consider the case where $\gba^\ast \in \calG$, and $\calG$ is closed under permutations.
For a symmetric prior $p_\rz$ (e.g., isotropic Gaussian), there exists an optimal solution $\gba^{\dagger}\in \calG$ to the \name{} objective (Eq.~\ref{eq:\name{}}) for $\lambda_{\ra}=\lambda_{\rb}=\infty$  such that $\gba^{\dagger} \neq \gba^\ast$. 
\end{proposition}

\begin{proof}
We will prove the proposition via contradiction. 
That is, let's assume that $\gba^\ast$ is a unique solution for the \name{} objective for $\lambda_{\ra}=\lambda_{\rb}=\infty$ (Eq.~\ref{eq:\name{}}). Now, consider an alternate mapping $\gba^{\dagger}=\gba^{\ast}S$ for an arbitrary non-identity permutation matrix $S\neq I$ in the symmetric group.

As before, we note that $\gba^{\ast}= \gza^{\ast} \circ{\gzb^{\ast^{-1}}}$ and $\gba^{\dagger}= \gza^{\dagger}\circ{\gzb^{\dagger^{-1}}}$ due to the invertibility constraints in Eqs.~\ref{eq:invert_1}-\ref{eq:invert_3}.
Since permutation matrices are invertible and so is $\gba^{\ast}$, their composition given by $\gba^{\dagger}$ is also invertible.
Further, since $\calG$ is closed under permutation and $\gba^{\ast} \in \calG$, we also have $\gba^{\dagger} \in \calG$.

Next, we note that the inverse of a permutation matrix is also a permutation matrix. Since the prior is assumed to be symmetric and a a transformation specified by a permutation matrix is volume-preserving (i.e., $\rm{det}(S) =1$ for all $S \in \calS_n$), we can use the change-of-variables formula in Eq.~\ref{eq:cov} to get:
\begin{align}
    \mathcal{L}_{\textnormal{MLE}}(\gza^{\ast}) &= \mathcal{L}_{\textnormal{MLE}}(\gza^{\dagger})\\
    \mathcal{L}_{\textnormal{MLE}}(\gzb^{\ast}) &= \mathcal{L}_{\textnormal{MLE}}(\gzb^{\dagger}).
\end{align}
Noting that $\gba^{\ast}= \gza^{\ast} \circ {\gzb^{\ast^{-1}}}$ and $\gba^{\dagger}= \gza^{\dagger}\circ {\gzb^{\dagger^{-1}}}$ due to the invertibility constraints in Eqs.~\ref{eq:invert_1}-\ref{eq:invert_3}, we can substitute the above equations in Eq.~\ref{eq:\name{}}. 
When $\lambda_{\ra}=\lambda_{\rb}=\infty$, for any choice of $C_\ra, C_\rb$ we have:
\begin{align}
    &\mathcal{L}_{\textnormal{\name{}}}(\gba^{\ast}, C_\ra, C_\rb, \lambda_{\ra}=\infty, \lambda_{\rb}=\infty)\nonumber\\
    &= \mathcal{L}_{\textnormal{\name{}}}(\gba^{\dagger}, C_\ra, C_\rb, \lambda_{\ra}=\infty, \lambda_{\rb}=\infty).
\end{align}

The above equation implies that $\gba^{\dagger}$ is also an optimal solution to the \name{} objective in Eq.~\ref{eq:\name{}} for $\lambda_{\ra}=\lambda_{\rb}=\infty$. 
Thus, we arrive at a contradiction since $\gba^{\ast}$ is not the unique maximizer. Hence, proved.
\end{proof}

The above construction suggests that MLE-only training can fail to identify the optimal mapping corresponding to the joint distribution $p^\ast_{\ra, \rb}$ even if it lies within the mappings represented via the family represented via the \name{} architecture.
Failure modes due to non-identifiability could also potentially arise for adversarial and hybrid training.
Empirically, we find that while MLE-only training gives poor performance for cross-domain translations, the hybrid and adversarial training objectives are much more effective, which suggests that these objectives are less susceptible to identifiability issues in recovering the true mapping.


\section{Experiment Details}\label{app:expt}
We used PyTorch~\citep{paszke2017automatic} for implementing our codebase. 
All models were trained on a single Nvidia TitanX GPU.



\subsection{Image-to-Image Translation}
We use the standard training, validation, and test splits for each dataset.
For datasets which do not provide a validation set (\emph{e.g.,} Facades and CityScapes), we randomly hold out a portion of the training set with the same number of images as the test set. 
We train each model for 200 epochs with a fixed learning rate of $2\cdot 10^{-4}$ for the first 100 epochs, followed by a linear decay schedule for 100 epochs from the initial learning rate to 0. We use the Adam
optimizer with $
\beta_1 = 0.5$ and $\beta_2 = 0.999$, and for \name{} we apply weight normalization
of $5\cdot 10^{-5}$ to the generator's parameters. 
When training with an MLE objective, we apply gradient clipping with a maximum gradient norm of 10. 
Scaling flow models to higher dimensionality is an active area of research; for this work we resized the images to $64\times 64$ for Cityscapes and Maps, and $128\times 128$ for Facades. 
We use a batch-size of 16 images.

For MLE/Hybrid models, we used an isotropic Gaussian prior.
Variation in performance as a function of $\lambda$ for the Maps dataset is shown in Table~\ref{table:lambda-results}.
Similar trends hold for other datasets.
We use the following flow architecture to parameterize $G_{\rz\to \ra}$ and $G_{\rz\to \rb}$:
\begin{align*}
&\text{Scale[Input: 32x32x3, Output: 16x16x6x2]} \\
&\rightarrow \text{3x CheckerboardCoupling[Channels: 32, Blocks: 4]}\\
&\rightarrow \text{3x ChannelwiseCoupling[Channels: 64, Blocks: 4]}\\
&\rightarrow \text{Squeeze\&Split[Input: 32x32x3, Output: 16x16x6x2]}\\
&\text{Scale[Input: 16x16x6, Output: 8x8x12x2]} \\
&\rightarrow \text{3x CheckerboardCoupling[Channels: 64, Blocks: 4]}\\
&\rightarrow \text{3x ChannelwiseCoupling[Channels: 128, Blocks: 4]}\\
&\rightarrow \text{Squeeze\&Split[Input: 16x16x6, Output: 8x8x12x2]}\\
&\text{Scale[Input: 8x8x12, Output: 4x4x24x2]} \\
&\rightarrow \text{3x CheckerboardCoupling[Channels: 128, Blocks: 4]}\\
&\rightarrow \text{3x ChannelwiseCoupling[Channels: 256, Blocks: 4]}\\
&\rightarrow \text{Squeeze\&Split[Input: 8x8x12, Output: 4x4x24x2]}\\
&\text{Scale[Input: 4x4x24, Output: 4x4x24]} \\
&\rightarrow \text{4x CheckerboardCoupling[Channels: 256, Blocks: 4]}
\end{align*}
where CheckerboardCoupling and ChannelwiseCoupling are affine coupling layers with checkerboard and channelwise masking, respectively, and where Squeeze\&Split first trades spatial extent for channels by turning each $4\times 4\times 1$ subvolume into a $1\times 1\times 4$ subvolume, and then splits the volume along the last dimension and sends half of the features directly to the latent space. See \citet{real-nvp} for more details. Within each affine coupling layer, we parametrize the scale and translate factors using a ResNet
architecture with the specified number of channels and residual blocks. 
We additionally use activation normalization~\cite{glow} before each coupling layer.

\begin{table}[t]
    \centering
    \caption{Test mean squared error (MSE) for domain translation on Maps dataset for different choices of $\lambda=\lambda_\textrm{A}=\lambda_\textrm{B}$ for the \name{} objective. Note the low values of $\lambda$ are due to the different scales of the MLE and adversarial losses.}
    \vspace{1em}
    \label{table:lambda-results}
    \begin{tabular}{@{}lcc@{}}
    \toprule
 Model  & MSE ($\textrm{A} \to \textrm{B}$)     & MSE ($\textrm{B} \to \textrm{A}$)      
   \\ \midrule
    Adversarial only, $\lambda=0$          & 0.0385	        & 0.1123         
    \\
    Hybrid, $\lambda=1e-5$           & \textbf{0.0209} & \textbf{0.0897} 
    \\
    Hybrid, $\lambda=1e-4$               & 0.0211 & 0.1362 
    \\
    Hybrid, $\lambda=1e-3$               & 0.0260 & 0.1381
    \\
    Hybrid, $\lambda=1e-2$              & 0.0561 & 0.1457
    \\
    MLE only, $\lambda=\infty$        & 0.0452          & 0.1746          
    \\ 
    \bottomrule
    \end{tabular}
\end{table}

\subsection{Unsupervised Domain Adaptation}\label{app:domain_adapt}

One such model relevant to this experiment is Cycle-Consistent Domain Adaptation (CyCADA)~\citep{cycada}. CyCADA first learns a cross-domain translation mapping from source to target domain via CycleGAN.
This mapping is used to stylize the source dataset into the target domain, which is then subject to additional feature-level and semantic consistency losses for learning the target domain classifier~\citep{ganin2014unsupervised,adda}. 
We direct the reader to \citet{cycada} for further details.

We use the same training, validation and test splits of MNIST, USPS, and SVHN digit datasets as in CyCADA~\citep{cycada}. 
For all datasets, images are resized to $32\times 32$ as in CyCADA.
We employ the pixel-level and feature-level adaptation training pipeline as in CyCADA but replace the CycleGAN-based image translation network with the \name{}. 
The architectures for imposing semantic consistency and feature adaptation are the same as the ones used for CyCADA.
The architecture and hyperparameter tuning protocol was consistent with the one used for image-to-image translations using \name{}. 
For the hyperparameters of feature-level domain adaptation post the image translations, we adopted the optimal hyperparameter settings from ADDA~\citep{adda}.

\end{appendix}

%% file: AAAI-GroverA.6849.bbl
\begin{thebibliography}{}

\bibitem[\protect\citeauthoryear{Alvarez-Melis, Jegelka, and
  Jaakkola}{2019}]{alvarez2018towards}
Alvarez-Melis, D.; Jegelka, S.; and Jaakkola, T.~S.
\newblock 2019.
\newblock Towards optimal transport with global invariances.
\newblock In {\em AISTATS}.

\bibitem[\protect\citeauthoryear{Arjovsky, Chintala, and
  Bottou}{2017}]{arjovsky2017wasserstein}
Arjovsky, M.; Chintala, S.; and Bottou, L.
\newblock 2017.
\newblock Wasserstein gan.
\newblock In {\em ICML}.

\bibitem[\protect\citeauthoryear{Bousmalis \bgroup et al\mbox.\egroup
  }{2017}]{bousmalis2017unsupervised}
Bousmalis, K.; Silberman, N.; Dohan, D.; Erhan, D.; and Krishnan, D.
\newblock 2017.
\newblock Unsupervised pixel-level domain adaptation with generative
  adversarial networks.
\newblock In {\em CVPR}.

\bibitem[\protect\citeauthoryear{Chen \bgroup et al\mbox.\egroup
  }{2018}]{chen2018neural}
Chen, T.~Q.; Rubanova, Y.; Bettencourt, J.; and Duvenaud, D.~K.
\newblock 2018.
\newblock Neural ordinary differential equations.
\newblock In {\em NeurIPS}.

\bibitem[\protect\citeauthoryear{Choi \bgroup et al\mbox.\egroup
  }{2018}]{choi2018stargan}
Choi, Y.; Choi, M.; Kim, M.; Ha, J.-W.; Kim, S.; and Choo, J.
\newblock 2018.
\newblock Stargan: Unified generative adversarial networks for multi-domain
  image-to-image translation.
\newblock In {\em CVPR}.

\bibitem[\protect\citeauthoryear{Cordts \bgroup et al\mbox.\egroup
  }{2016}]{cityscapes}
Cordts, M.; Omran, M.; Ramos, S.; Rehfeld, T.; Enzweiler, M.; Benenson, R.;
  Franke, U.; Roth, S.; and Schiele, B.
\newblock 2016.
\newblock The cityscapes dataset for semantic urban scene understanding.
\newblock In {\em CVPR}.

\bibitem[\protect\citeauthoryear{Courty \bgroup et al\mbox.\egroup
  }{2017}]{courty2017joint}
Courty, N.; Flamary, R.; Habrard, A.; and Rakotomamonjy, A.
\newblock 2017.
\newblock Joint distribution optimal transportation for domain adaptation.
\newblock In {\em NeurIPS}.

\bibitem[\protect\citeauthoryear{Cui \bgroup et al\mbox.\egroup
  }{2014}]{cui2014generalized}
Cui, Z.; Chang, H.; Shan, S.; and Chen, X.
\newblock 2014.
\newblock Generalized unsupervised manifold alignment.
\newblock In {\em NeurIPS}.

\bibitem[\protect\citeauthoryear{Danihelka \bgroup et al\mbox.\egroup
  }{2017}]{danihelka2017comparison}
Danihelka, I.; Lakshminarayanan, B.; Uria, B.; Wierstra, D.; and Dayan, P.
\newblock 2017.
\newblock Comparison of maximum likelihood and gan-based training of real nvps.
\newblock {\em arXiv preprint arXiv:1705.05263}.

\bibitem[\protect\citeauthoryear{Dinh, Krueger, and Bengio}{2014}]{nice}
Dinh, L.; Krueger, D.; and Bengio, Y.
\newblock 2014.
\newblock Nice: Non-linear independent components estimation.
\newblock {\em arXiv preprint arXiv:1410.8516}.

\bibitem[\protect\citeauthoryear{Dinh, Sohl-Dickstein, and
  Bengio}{2017}]{real-nvp}
Dinh, L.; Sohl-Dickstein, J.; and Bengio, S.
\newblock 2017.
\newblock Density estimation using real nvp.
\newblock {\em arXiv preprint arXiv:1605.08803}.

\bibitem[\protect\citeauthoryear{Galanti, Wolf, and
  Benaim}{2017}]{galanti2017role}
Galanti, T.; Wolf, L.; and Benaim, S.
\newblock 2017.
\newblock The role of minimal complexity functions in unsupervised learning of
  semantic mappings.
\newblock {\em arXiv preprint arXiv:1709.00074}.

\bibitem[\protect\citeauthoryear{Ganin and
  Lempitsky}{2014}]{ganin2014unsupervised}
Ganin, Y., and Lempitsky, V.
\newblock 2014.
\newblock Unsupervised domain adaptation by backpropagation.
\newblock {\em arXiv preprint arXiv:1409.7495}.

\bibitem[\protect\citeauthoryear{Goodfellow \bgroup et al\mbox.\egroup
  }{2014}]{gan}
Goodfellow, I.; Pouget-Abadie, J.; Mirza, M.; Xu, B.; Warde-Farley, D.; Ozair,
  S.; Courville, A.; and Bengio, Y.
\newblock 2014.
\newblock Generative adversarial nets.
\newblock In {\em NeurIPS}.

\bibitem[\protect\citeauthoryear{Goodfellow}{2016}]{goodfellow2016nips}
Goodfellow, I.
\newblock 2016.
\newblock Nips 2016 tutorial: Generative adversarial networks.
\newblock {\em arXiv preprint arXiv:1701.00160}.

\bibitem[\protect\citeauthoryear{Grathwohl \bgroup et al\mbox.\egroup
  }{2018}]{grathwohl2018ffjord}
Grathwohl, W.; Chen, R.~T.; Betterncourt, J.; Sutskever, I.; and Duvenaud, D.
\newblock 2018.
\newblock Ffjord: Free-form continuous dynamics for scalable reversible
  generative models.
\newblock {\em arXiv preprint arXiv:1810.01367}.

\bibitem[\protect\citeauthoryear{Grover, Dhar, and Ermon}{2018}]{flow-gan}
Grover, A.; Dhar, M.; and Ermon, S.
\newblock 2018.
\newblock Flow-gan: Combining maximum likelihood and adversarial learning in
  generative models.
\newblock In {\em AAAI Conference on Artificial Intelligence}.

\bibitem[\protect\citeauthoryear{Gu \bgroup et al\mbox.\egroup
  }{2018}]{gu2018universal}
Gu, J.; Hassan, H.; Devlin, J.; and Li, V.~O.
\newblock 2018.
\newblock Universal neural machine translation for extremely low resource
  languages.
\newblock {\em arXiv preprint arXiv:1802.05368}.

\bibitem[\protect\citeauthoryear{Ho \bgroup et al\mbox.\egroup
  }{2019}]{ho2019flow++}
Ho, J.; Chen, X.; Srinivas, A.; Duan, Y.; and Abbeel, P.
\newblock 2019.
\newblock Flow++: Improving flow-based generative models with variational
  dequantization and architecture design.
\newblock In {\em ICML}.

\bibitem[\protect\citeauthoryear{Hoffman \bgroup et al\mbox.\egroup
  }{2017}]{cycada}
Hoffman, J.; Tzeng, E.; Park, T.; Zhu, J.-Y.; Isola, P.; Saenko, K.; Efros,
  A.~A.; and Darrell, T.
\newblock 2017.
\newblock Cycada: Cycle-consistent adversarial domain adaptation.
\newblock {\em arXiv preprint arXiv:1711.03213}.

\bibitem[\protect\citeauthoryear{Hoshen and Wolf}{2018}]{hoshen2018nam}
Hoshen, Y., and Wolf, L.
\newblock 2018.
\newblock Nam: Non-adversarial unsupervised domain mapping.
\newblock In {\em ECCV}.

\bibitem[\protect\citeauthoryear{Hoshen}{2018}]{hoshen2018non}
Hoshen, Y.
\newblock 2018.
\newblock Non-adversarial mapping with vaes.
\newblock In {\em NeurIPS}.

\bibitem[\protect\citeauthoryear{Huang \bgroup et al\mbox.\egroup
  }{2018a}]{huang2018neural}
Huang, C.-W.; Krueger, D.; Lacoste, A.; and Courville, A.
\newblock 2018a.
\newblock Neural autoregressive flows.
\newblock {\em arXiv preprint arXiv:1804.00779}.

\bibitem[\protect\citeauthoryear{Huang \bgroup et al\mbox.\egroup
  }{2018b}]{huang2018multimodal}
Huang, X.; Liu, M.-Y.; Belongie, S.; and Kautz, J.
\newblock 2018b.
\newblock Multimodal unsupervised image-to-image translation.
\newblock In {\em ECCV}.

\bibitem[\protect\citeauthoryear{Isola \bgroup et al\mbox.\egroup
  }{2017}]{isola2017image}
Isola, P.; Zhu, J.-Y.; Zhou, T.; and Efros, A.~A.
\newblock 2017.
\newblock Image-to-image translation with conditional adversarial networks.
\newblock In {\em CVPR}.

\bibitem[\protect\citeauthoryear{Kim \bgroup et al\mbox.\egroup
  }{2017}]{kim2017learning}
Kim, T.; Cha, M.; Kim, H.; Lee, J.~K.; and Kim, J.
\newblock 2017.
\newblock Learning to discover cross-domain relations with generative
  adversarial networks.
\newblock {\em arXiv preprint arXiv:1703.05192}.

\bibitem[\protect\citeauthoryear{Kingma and Dhariwal}{2018}]{glow}
Kingma, D.~P., and Dhariwal, P.
\newblock 2018.
\newblock Glow: Generative flow with invertible 1x1 convolutions.
\newblock {\em arXiv preprint arXiv:1807.03039}.

\bibitem[\protect\citeauthoryear{Kingma \bgroup et al\mbox.\egroup
  }{2016}]{iaf}
Kingma, D.~P.; Salimans, T.; Jozefowicz, R.; Chen, X.; Sutskever, I.; and
  Welling, M.
\newblock 2016.
\newblock Improved variational inference with inverse autoregressive flow.
\newblock In {\em NeurIPS}.

\bibitem[\protect\citeauthoryear{Kumar \bgroup et al\mbox.\egroup
  }{2018}]{kumar2018co}
Kumar, A.; Sattigeri, P.; Wadhawan, K.; Karlinsky, L.; Feris, R.; Freeman, B.;
  and Wornell, G.
\newblock 2018.
\newblock Co-regularized alignment for unsupervised domain adaptation.
\newblock In {\em NeurIPS}.

\bibitem[\protect\citeauthoryear{Liu and Tuzel}{2016}]{liu2016coupled}
Liu, M.-Y., and Tuzel, O.
\newblock 2016.
\newblock Coupled generative adversarial networks.
\newblock In {\em NeurIPS}.

\bibitem[\protect\citeauthoryear{Liu \bgroup et al\mbox.\egroup
  }{2018a}]{liu2018unified}
Liu, A.~H.; Liu, Y.-C.; Yeh, Y.-Y.; and Wang, Y.-C.~F.
\newblock 2018a.
\newblock A unified feature disentangler for multi-domain image translation and
  manipulation.
\newblock In {\em NeurIPS}.

\bibitem[\protect\citeauthoryear{Liu \bgroup et al\mbox.\egroup
  }{2018b}]{liu2018detach}
Liu, Y.-C.; Yeh, Y.-Y.; Fu, T.-C.; Wang, S.-D.; Chiu, W.-C.; and Frank~Wang,
  Y.-C.
\newblock 2018b.
\newblock Detach and adapt: Learning cross-domain disentangled deep
  representation.
\newblock In {\em CVPR}.

\bibitem[\protect\citeauthoryear{Liu, Breuel, and
  Kautz}{2017}]{liu2017unsupervised}
Liu, M.-Y.; Breuel, T.; and Kautz, J.
\newblock 2017.
\newblock Unsupervised image-to-image translation networks.
\newblock In {\em NeurIPS}.

\bibitem[\protect\citeauthoryear{Long \bgroup et al\mbox.\egroup
  }{2018}]{long2018conditional}
Long, M.; Cao, Z.; Wang, J.; and Jordan, M.~I.
\newblock 2018.
\newblock Conditional adversarial domain adaptation.
\newblock In {\em NeurIPS}.

\bibitem[\protect\citeauthoryear{Mirza and
  Osindero}{2014}]{mirza2014conditional}
Mirza, M., and Osindero, S.
\newblock 2014.
\newblock Conditional generative adversarial nets.
\newblock {\em arXiv preprint arXiv:1411.1784}.

\bibitem[\protect\citeauthoryear{Mohamed and
  Lakshminarayanan}{2016}]{mohamed2016learning}
Mohamed, S., and Lakshminarayanan, B.
\newblock 2016.
\newblock Learning in implicit generative models.
\newblock {\em arXiv preprint arXiv:1610.03483}.

\bibitem[\protect\citeauthoryear{Nowozin, Cseke, and
  Tomioka}{2016}]{nowozin2016f}
Nowozin, S.; Cseke, B.; and Tomioka, R.
\newblock 2016.
\newblock f-gan: Training generative neural samplers using variational
  divergence minimization.
\newblock In {\em NeurIPS}.

\bibitem[\protect\citeauthoryear{Papamakarios, Murray, and
  Pavlakou}{2017}]{maf}
Papamakarios, G.; Murray, I.; and Pavlakou, T.
\newblock 2017.
\newblock Masked autoregressive flow for density estimation.
\newblock In {\em NeurIPS}.

\bibitem[\protect\citeauthoryear{Paszke \bgroup et al\mbox.\egroup
  }{2017}]{paszke2017automatic}
Paszke, A.; Gross, S.; Chintala, S.; Chanan, G.; Yang, E.; DeVito, Z.; Lin, Z.;
  Desmaison, A.; Antiga, L.; and Lerer, A.
\newblock 2017.
\newblock Automatic differentiation in pytorch.

\bibitem[\protect\citeauthoryear{Rezende and Mohamed}{2015}]{normalizing-flow}
Rezende, D.~J., and Mohamed, S.
\newblock 2015.
\newblock Variational inference with normalizing flows.
\newblock {\em arXiv preprint arXiv:1505.05770}.

\bibitem[\protect\citeauthoryear{Saenko \bgroup et al\mbox.\egroup
  }{2010}]{saenko2010adapting}
Saenko, K.; Kulis, B.; Fritz, M.; and Darrell, T.
\newblock 2010.
\newblock Adapting visual category models to new domains.
\newblock In {\em ECCV}.

\bibitem[\protect\citeauthoryear{Sankaranarayanan \bgroup et al\mbox.\egroup
  }{2018}]{sankaranarayanan2018generate}
Sankaranarayanan, S.; Balaji, Y.; Castillo, C.~D.; and Chellappa, R.
\newblock 2018.
\newblock Generate to adapt: Aligning domains using generative adversarial
  networks.
\newblock In {\em CVPR}.

\bibitem[\protect\citeauthoryear{Shu \bgroup et al\mbox.\egroup
  }{2018}]{shu2018dirt}
Shu, R.; Bui, H.~H.; Narui, H.; and Ermon, S.
\newblock 2018.
\newblock A dirt-t approach to unsupervised domain adaptation.
\newblock In {\em ICLR}.

\bibitem[\protect\citeauthoryear{Taigman, Polyak, and
  Wolf}{2016}]{taigman2016unsupervised}
Taigman, Y.; Polyak, A.; and Wolf, L.
\newblock 2016.
\newblock Unsupervised cross-domain image generation.
\newblock {\em arXiv preprint arXiv:1611.02200}.

\bibitem[\protect\citeauthoryear{Tzeng \bgroup et al\mbox.\egroup
  }{2017}]{adda}
Tzeng, E.; Hoffman, J.; Saenko, K.; and Darrell, T.
\newblock 2017.
\newblock Adversarial discriminative domain adaptation.
\newblock In {\em CVPR}.

\bibitem[\protect\citeauthoryear{Wang \bgroup et al\mbox.\egroup
  }{2018}]{wang2018video}
Wang, T.-C.; Liu, M.-Y.; Zhu, J.-Y.; Liu, G.; Tao, A.; Kautz, J.; and
  Catanzaro, B.
\newblock 2018.
\newblock Video-to-video synthesis.
\newblock In {\em NeurIPS}.

\bibitem[\protect\citeauthoryear{Wu \bgroup et al\mbox.\egroup
  }{2019}]{wu2019domain}
Wu, Y.; Winston, E.; Kaushik, D.; and Lipton, Z.
\newblock 2019.
\newblock Domain adaptation with asymmetrically-relaxed distribution alignment.
\newblock {\em arXiv preprint arXiv:1903.01689}.

\bibitem[\protect\citeauthoryear{Yi \bgroup et al\mbox.\egroup
  }{2017}]{yi2017dualgan}
Yi, Z.; Zhang, H.; Tan, P.; and Gong, M.
\newblock 2017.
\newblock Dualgan: Unsupervised dual learning for image-to-image translation.
\newblock In {\em ICCV}.

\bibitem[\protect\citeauthoryear{Zhu \bgroup et al\mbox.\egroup
  }{2017a}]{cycle-gan}
Zhu, J.-Y.; Park, T.; Isola, P.; and Efros, A.~A.
\newblock 2017a.
\newblock Unpaired image-to-image translation using cycle-consistent
  adversarial networks.
\newblock In {\em ICCV}.

\bibitem[\protect\citeauthoryear{Zhu \bgroup et al\mbox.\egroup
  }{2017b}]{zhu2017toward}
Zhu, J.-Y.; Zhang, R.; Pathak, D.; Darrell, T.; Efros, A.~A.; Wang, O.; and
  Shechtman, E.
\newblock 2017b.
\newblock Toward multimodal image-to-image translation.
\newblock In {\em NeurIPS}.

\end{thebibliography}
